\documentclass{article}
\usepackage[numbers]{natbib}

\usepackage[utf8]{inputenc} 
\usepackage[T1]{fontenc}    
\usepackage{hyperref}       
\usepackage{url}            
\usepackage{booktabs}       
\usepackage{amsfonts}       
\usepackage{nicefrac}       
\usepackage{microtype}      
\usepackage{xcolor}         
\usepackage{amsthm}
\newtheorem{definition}{Definition}
\newtheorem{assumption}{Assumption}

\newtheorem{theorem}{Theorem}
\newtheorem{lemma}{Lemma}

\usepackage[linesnumbered, ruled, lined,boxed,commentsnumbered]{algorithm2e}[1]
\usepackage{amsmath}
\usepackage{pdfpages}

     \usepackage[preprint]{neurips_2024}

\newcommand{\bsx}{\mathbf{x}}
\newcommand{\mbf}{\mathbf}
\newcommand{\sbf}{\boldsymbol}

\title{High dimensional Bayesian Optimization via Condensing-Expansion Projection}

\author{%
    Jiaming Lu \\
    ISTBI, Fudan University\\
    Shanghai, China \\
   \And
  Rong J.B. Zhu\thanks{Corresponding: \texttt{rongzhu@fudan.edu.cn}} \\
  ISTBI, Fudan University\\
  Shanghai, China\\
}

\begin{document}

\maketitle

\begin{abstract}
In high-dimensional settings, Bayesian optimization (BO) can be expensive and infeasible. The random embedding Bayesian optimization algorithm is commonly used to address high-dimensional BO challenges. However, this method relies on the effective subspace assumption on the optimization problem's objective function, which limits its applicability. 
In this paper, we introduce Condensing-Expansion Projection Bayesian optimization (CEPBO), a novel random projection-based approach for high-dimensional BO that does not reply on the effective subspace assumption. The approach is both simple to implement and highly practical. 
We present two algorithms based on different random projection matrices: the Gaussian projection matrix and the hashing projection matrix. 
Experimental results demonstrate that both algorithms outperform existing random embedding-based algorithms in most cases, achieving superior performance on high-dimensional BO problems.
The code is available in \url{https://anonymous.4open.science/r/CEPBO-14429}.
\end{abstract}

\section{Introduction}
For many optimization problems, the objective function $f$ lacks a closed-form expression, and gradient information is often unavailable, leading to what we are generally referred to as black-box functions \cite{jones1998efficient,snoek2012practical,shahriari2015taking}.
Bayesian optimization (BO) is an efficient method for solving such optimization problems by modeling the unknown objective function through a probabilistic surrogate model, typically a Gaussian Process. The BO routine is a sequential search algorithm where each iteration involves estimating the surrogate model from available data and then maximizing an acquisition function to determine which point should be evaluated next. 
As the input space dimension $D$ increases, typically $D\geq 10$,  BO encounters the so-called `curse of dimensionality' \cite{bellman1966dynamic}. This phenomenon refers to the  exponential increase in difficulty resulting from higher query complexity and the computational cost associated with calculating the acquisition function. 

To address the issue, numerous studies have proposed high-dimensional BO algorithms \citep{rembo,chen2012joint,binois2015warped,binois2020choice,hesbo,alebo} that typically translate high-dimensional optimizations into low-dimensional ones by various techniques, and search the new point in the low-dimensional space.  
However, these approaches can become inefficient when the maximum over the high-dimensional space cannot be well approximated by the maximum over the low-dimensional space.  

In this paper, we introduce a novel search strategy in high-dimensional BO problems called the Condense-Expansion Projection (CEP) technique, which is both simple to implement and highly practical. 
In each iteration of the sequential search, the CEP technique generates
a random projection matrix $\mathbf{A}\in \mathbb{R}^{d\times D}$, where $d\ll D$, to project the available data from the high-dimensional space to the low-dimensional embedding space by multiplying them with $\mathbf{A}$. It estimates the surrogate model and searches for the next point to evaluation in the low-dimensional embedding space. Subsequently, it projects the searched data point back to the high-dimensional space by multiplying it with $\mathbf{A}^{\top}$ for evaluation in the original space.  

We utilize two distinct random projection matrices: the Gaussian projection matrix \cite{DG:2002} and the hashing projection matrix \cite{RT:2008,BG:2013}, within the CEP technique, 
resulting in the development of two algorithms, CEP-REMBO and CEP-HeSBO. 
Our experimental results, comprising comprehensive simulation studies and analysis of four real-world datasets, demonstrate that both algorithms generally outperform existing random embedding-based algorithms, showcasting the superior performance of the CEP technique on high-dimensional BO problems.

\section{Related Work}
There is a substantial body of literature on high-dimensional BO.
The most closely related approach is REMBO  \cite{rembo} by fitting a Gaussian Process model in a low-dimensional embedding space obtained through a Gaussian random projection matrix. This approach has been further investigated by \cite{binois2015warped,binois2020choice,binois2015uncertainty,alebo} under various conditions. 
\cite{hesbo} proposed HeSBO that utilizes a hashing projection matrix. 
However, these studies are based on the assumption of an effective subspace, where a small number of parameters have a significant impact on the objective function. 
Similar to these studies, our approach evaluates the acquisition function over the embedding space. However, unlike these studies, our approach select the new point in the original space, implying that our approach does not relies on the effective space assumption. The second distinguishing aspect of our approach is that it generates a new random projection matrix in each iteration, thereby increasing the flexibility to accommodate the possibility that the global optimum may not be located in a single embedding space. 

Besides the embedding approach, two other techniques warrant consideration: one based on the additive form of the objective function, and the other based on the dropout approach.
\cite{kandasamy2015high} tackled the challenges in high-dimensional BO by assuming an additive structure for the function. 
Other works along the line include GPs with an additive kernel \cite{MK:2018,WLJK:2017} or cylindrical kernels \cite{BOCK}.
However, this approach is limited by its reliance on the assumption of the additive form of the objective function. 
\cite{li2018high} applied the dropout technique into high-dimensional BO to alleviate reliance on assumptions regarding limited “active” features or the additive form of the objective function.
This method randomly selects subset of dimensions and optimizes variables only from these chosen dimensions via Bayesian optimization. 
However, it necessitates ``filling-in"  the variables from the left-out dimensions. 
The proposed ``fill-in" strategy, which involves copying the value of variables from the best function value, may lead to being trapped in a local optimum, although the strategy is enhanced by mixing random values.

\section{Method}

\subsection{Bayesian Optimization}
We consider the optimization problem 
$$\mbf{x}^* = \arg \max\nolimits_{\mbf{x} \in \mathcal{X}}  f(\mbf{x}),$$ 
where $f$ is a black-box function
and $\mathcal{X}\subset \mathbb{R}^D$ is some bounded set. 
BO is a form of sequential model-based optimization, where we fit a surrogate model, typically a Gaussian Process (GP) model, for $f$ that is used to identify which parameters $\mbf{x}$ should be evaluated next. 
The GP surrogate model is $f\sim \mathcal{GP}\left(m(\cdot), k(\cdot,\cdot)\right)$, with a mean function $m(\cdot)$ and a kernel $k(\cdot,\cdot)$. 
Under the GP prior, the posterior for the value of $f(\mbf{x})$ at any point in the space is a normal distribution with closed-form mean and variance. Using that posterior, we
construct an acquisition function $\alpha(\mbf{x})$ that specifies the utility of evaluating $f$ at $\mbf{x}$, such as Expected Improvement \cite{jones1998efficient}. 
We find $\mbf{x}_{\text{new}} = \arg \max\nolimits_{\mbf{x} \in \mathcal{X}}  \alpha(\mbf{x})$, and in the next iteration evaluate $f(\mbf{x}_{\text{new}})$.
However, GPs are known to predict poorly for large dimension $D$ \cite{rembo}, which prevents the use of standard BO in high dimensions.

\subsection{Condensing-Expansion Projection}
The framework of Condensing-Expansion Projection (CEP) is delineated through two essential projection procedures.
\begin{itemize}
  \item \textbf{Condensing Projection:} transpose points from the original space into a reduced-dimensional embedding subspace, where the surrogate model is fitted from available data and the acquisition function is maximized to determine which point should be evaluated next. 
  \item \textbf{Expansion Projection:} revert these points in the embedding subspace back to the original space, where the searched point is evaluated. 
\end{itemize}

\begin{figure*}[h]
    \centering
    \includegraphics[width=1\textwidth]{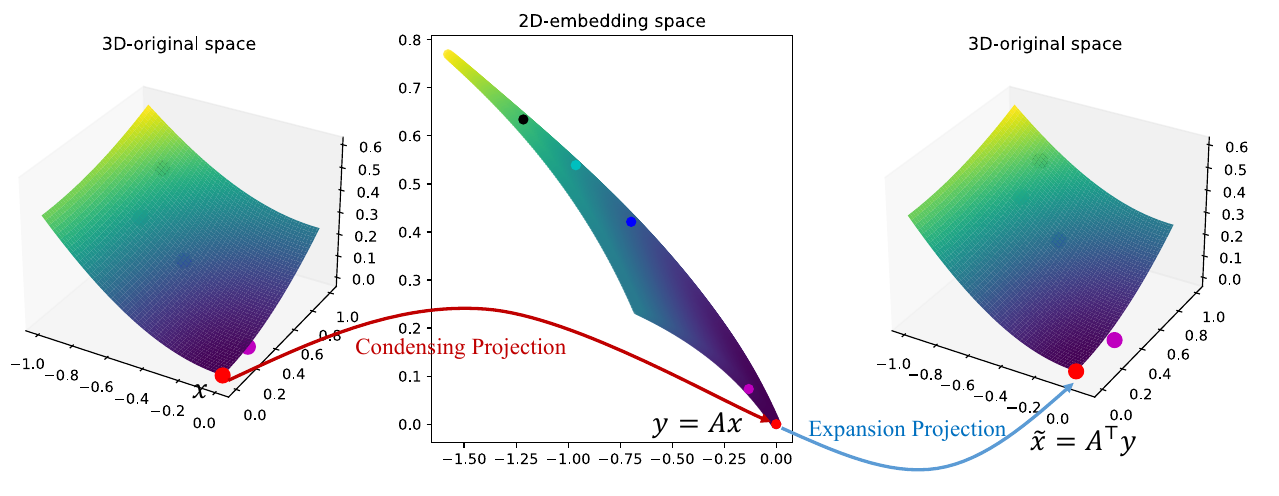}
    \caption{
    An illustration of CEP. The dimension of the original space is $D=3$, and the dimension of the embedding subspace is $d=2$. 
    The five points in the orginal space is projected to the embedding subspace by Condensing Projection, then they are projected back to the original space by Expansion Projection. 
    The optimal point (red dot) in the original space is still at the (approximately) optimal position after CEP.
    }
    \label{fig: cep illustraction}
\end{figure*}

Let us define an embedding subspace $\mathcal{Y} \subset \mathbb{R}^d$ of dimension $d$. 
We generate a random projection matrix $\mathbf{A} \in \mathbb{R}^{d\times D}$. 
Various methodologies exist for the construction of such a matrix, including the Gaussian random matrix, sparse random matrix \cite{dasgupta2013experiments,bingham2001random}, and the Subsampled Randomized Hadamard  Transform \cite{tropp2011improved}.
In this paper, we utilize the Gaussian random matrix and the Hashing matrix. 

Consider a point $\mbf{x} \in \mathcal{X}$ within the original space $\mathcal{X}$.
In the condensing projection, we project $\mbf{x}$ from the original space $\mathcal{X}$ to the embedding subspace $\mathcal{Y}$ by multiplying $\mbf{x}$ with the matrix $\mathbf{A}$, 
resulting in 
$$\mbf{y} = \mathbf{A} \mbf{x} \in \mathcal{Y},$$ 
thereby reducing the dimension from $D$ in the original space to $d$ in the embedding subspace. 
In the expansion projection, we project $\mbf{y}$ back to $\mathcal{X}$ by multiplying $\mbf{y}$ with the transposed matrix $\mathbf{A}^\top$, expressed as 
as 
$$\tilde{\mbf{x}} = \mathbf{A}^\top \mbf{y} = \mathbf{A}^\top\mathbf{A} \mbf{x}.$$
This completes the Condensing-Expansion Projection, 
which can be outlined as follows: 
transforming a point from the original space to the embedding subspace and then restoring it back to the 
original space, represented as 
$$\mbf{x} \rightarrow \mbf{y} \rightarrow \tilde{\mbf{x}}.$$
We outline an illustration in Figure\ref{fig: cep illustraction}.

The CEP offers flexibility in selecting random projection matrix $\mathbf{A}$. 
In this paper, we focus on two types: Gaussian random matrices  
and Hashing random matrices. 


\begin{definition}{(Gaussian Random Matrix)}
    Let $\mathbf{A} \in \mathbb{R}^{d \times D}$ be a random matrix with independent Gaussian entries.
    For any $1 \leq i \leq d$ and $1 \leq j \leq D$, the element $a_{ij}$ defined 
    as $$a_{ij} \sim \mathcal{N}\left(0,\frac{1}{d}\right).$$
    \label{random Gaussian matrix definition}
\end{definition}
\begin{definition}{(Hashing Random Matrix)}
    Let $\mathbf{A} \in \mathbb{R}^{d \times D}$ be a hashing random matrix. 
    Specifically,
    (1) Each column of $\mathbf{A}$ has a single non-zero element that is selected at random. (2) This non-zero element has an equal probability $p=0.5$ of being either $+1$ or $-1$.
    \label{hashing random projection matrix}
\end{definition}

\begin{theorem}
    Assume a matrix $\mathbf{A}$ 
    satisfying Definition \ref{random Gaussian matrix definition} or Definition \ref{hashing random projection matrix}, 
    it can be shown that
    $$\mathbb{E}\left[\mathbf{A}^\top\mathbf{A}\right] = \mathbf{I}_D.$$
    \label{expectation_properties}
\end{theorem}

The proof of Theorem \ref{expectation_properties} is provided in the appendix \ref{Proof of Property expectation_properties}. 
Theorem \ref{expectation_properties} represents the isometry in expectation, 
suggesting that the process of two projections, on average, preserves both the magnitude and direction of $\mathbf{x}$, thereby making the gap from $\tilde{\mathbf{x}}$ to  $\mathbf{x}$ in the original space is zero on average.


Besides of the isometry in expectation, we also concern the concentration of $\tilde{\mbf{x}}$ around the original point $\mbf{x}$ in terms of the function $f$. 
Assume that there exists a vector $\mbf{g}\in\mathbb{R}^D$ such that, for $\bsx_1, \bsx_2\in \mathcal{X}$ are within a neighborhood, i.e., $\bsx_2\in N_{\epsilon}(\bsx_1)$ where $N_{\epsilon}(\bsx_1)=\{\bsx\in \mathcal{X}: |\bsx-\bsx_1|<\epsilon\}$, $|f(\bsx_1)-f(\bsx_2)| \leq |\mbf{g}^{\top}(\bsx_1 - \bsx_2)|$. Here $\mbf{g}$ represents an upper bound on the rate of change of $f$ within the neighborhood.
Therefore the disparity between $f(\bsx)$ and $f(\tilde{\bsx})$ can be bounded by $|\mbf{g}^{\top}(\bsx_1 - \bsx_2)|$. 
The following theorem provides the concentration of $\mbf{g}^{\top}\tilde{\bsx}$ round $\mbf{g}^{\top}\bsx$.

\begin{theorem}
    (1) Assume a matrix $\mathbf{A}$ 
    satisfying Definition \ref{random Gaussian matrix definition},
    then for any vector $\bsx \in \mathcal{X}$ and any vector $\mbf{g} \in \mathbb{R}^D$, it holds that
    $$\mathbb{E}\left[(\mbf{g}^\top \mathbf{A}^\top\mathbf{A}\bsx-\mbf{g}^\top\bsx)^2\right] \leq \frac{3}{d}\|\mbf{x}\|^2\|\mbf{g}\|^2-\frac{1}{d}(\mbf{g}^\top\mbf{x})^2.$$
    (2) Assume a matrix $\mathbf{A}$ 
    satisfying Definition \ref{hashing random projection matrix},
    then for any vector $\bsx \in \mathcal{X}$ and any vector $\mbf{g} \in \mathbb{R}^D$, it holds that
    $$\mathbb{E}\left[(\mbf{g}^\top \mathbf{A}^\top\mathbf{A}\bsx-\mbf{g}^\top\bsx)^2\right] \leq 
    \frac{1}{d}\left(\|\mbf{g}\|^2\|\mbf{x}\|^2-\sum\limits_{i=1}^Dx_i^2g_i^2\right).$$
    \label{prop:Variance}
\end{theorem}

The proof of Theorem \ref{prop:Variance} is provided in the appendix \ref{Proof of Property Variance}. 
Theorem \ref{prop:Variance} states that the variance of $\mbf{g}^{\top}\tilde{\bsx}$ with respect to $\mbf{g}^{\top}\bsx$ diminishes as the dimension $d$ of the embedding subspace increases. 
Hence, CEP can maintain the function in the original space well when $d$ is not small. 

\subsection{The CEPBO Algorithm}

We employ Condensing-Expansion Projection in Bayesian Optimization, leading to the development of the Condensing-Expansion Projection Bayesian Optimization (CEPBO) algorithm. 
In contrast to Random Embedding algorithms, sucha as REMBO\cite{rembo}, HeSBO\cite{hesbo} and ALEBO\cite{alebo}, where a fixed projection matrix is employed, 
the CEPBO algorithm dynamically generates a new projection matrix $\mathbf{A}_t$ during each iteration $t$. 
Through Condensing Projection, which condenses available points from the original space to a new embedding subspace via multiplication with $\mathbf{A}_t$, CEPBO leverages past information to conduct BO within the embedding subspace. It then determines which point to evaluate next within this subspace. 
Afterward, the selected point in the embedding subspace undergoes Expansion Projection, where it is projected back to the original space via multiplication with $\mathbf{A}_t^{\top}$. Subsequently, the objective function is evaluated at that point. 
This approach bypasses the stringent assumption associated with effective dimension by facilitating the projection of a new embedding subspace with each search iteration.
This adaptability acknowledges the possibility that the global optimum may not be confined to a single search. 

The detailed procedural flow of the algorithm is outlined in Algorithm\ref{CEPBO Algorithm}.
By using different random projection matrices at line 4 of the Algorithm\ref{CEPBO Algorithm}, we derive two algorithms: CEP-REMBO and CEP-HeSBO. These can be regarded as enhanced versions of REMBO and HeSBO, respectively.

\textbf{Condense original space into the embedding subspace.}\; 
The core concept of employing Condensing Projection entails the creation of a new subspace $\mathcal{Y}$ with each iteration, where BO is subsequently performed.
However, as the historical trajectories remains preserved within the original space $\mathcal{X}$, the newly formed subspace must be equipped with the requisite information to enable effective BO. 
To tackle this issue, the primary objective of Condensing Projection is to transfer the historical trajectories from the original space $\mathcal{X}$ into an embedding subspace $\mathcal{Y}$, thereby furnishing  
the embedding subspace $\mathcal{Y}$ with the necessary information to facilitate BO.
At the current iteration $t$, let $\mathcal{D}_{t-1}$ represent the trajectories in the original space $\mathcal{X}$, given by: 
$
\mathcal{D}_{t-1} = \{(\mbf{x}_1,f(\mbf{x}_1)), (\mbf{x}_2,f(\mbf{x}_2)), \ldots, (\mbf{x}_{t-1},f(\mbf{x}_{t-1}))\}.
$
During this iteration, a new projection matrix $\mathbf{A}_t$ of dimensions $\mathbb{R}^{d \times D}$ is sampled. 
This matrix serves as projecting the historical point from the original space $\mathcal{X}$ into a newly formed embedding subspace $\mathcal{Y}_t$:
$
\mathcal{D}_{t-1}^y = \{(\mathbf{A}_t\mbf{x}_1,f(\mbf{x}_1)), (\mathbf{A}_t\mbf{x}_2,f(\mbf{x}_2)), \cdots, (\mathbf{A}_t\mbf{x}_{t-1},f(\mbf{x}_{t-1}))\}.
$

\textbf{Optimize over the embedding subspace.}\; 
The objective $f$ is fitted by a Gaussian Process model over the embedding subspace $\mathcal{Y}_t$: $$f(\mbf{x})\approx \mathcal{GP}\left(m(\mathbf{A}_t\mbf{x}),k(\mathbf{A}_t\mbf{x},\mathbf{A}_t\mbf{x})\right).$$
Within the embedding subspace $\mathcal{Y}_t$, the dataset $\mathcal{D}_{t-1}^y$ informs the estimation of the hyperparameters $\theta_t$ for the Gaussian process, and the posterior probability of the Gaussian process is computed. 
The acquisition function $\alpha$ (such as Expected Improvement) identifies the embedding subspace's optimal point $\mbf{y}_t^*$ within the embedding subspace, which is represented by the equation
$
\mbf{y}_t^* = \underset{\mbf{y} \in \mathcal{Y}_t}{\arg \max} \ \alpha(\mbf{y} \,|\, \mathcal{D}_{t-1}^y).
$

\textbf{Project back and evaluate in the original space.}\; 
After searching the optimal point with the acquisition function, we need to project this point back to the original space $\mathcal{X}$ for objective function evaluation.  Subsequently, we add this point to the historical trajectories. 
To be more specific, we use the transpose projection matrix $\mathbf{A}_t$ to map the optimal point $\mbf{y}_t^*$ from the embedding subspace back 
to $\mathcal{X}$ by applying its transpose $\mathbf{A}_t^\top$, expressed as: 
$
\mbf{x}_t^* = \mathbf{A}_t^\top\mbf{y}_t^*.
$
Subsequently, we evaluate the objective function at $\mbf{x}_t^*$ within $\mathcal{X}$ to obtain $f(\mbf{x}_t^*)$. This data, denoted as $(\mbf{x}_t^*, f(\mbf{x}_t^*))$, is then added to the historical trajectories $\mathcal{D}_{t-1}$, resulting in: 
$
\mathcal{D}_{t} = \mathcal{D}_{t-1} \cup \{(\mbf{x}_t^*, f(\mbf{x}_t^*))\}.
$
This completes a full iteration cycle of the CEPBO algorithm.

\begin{algorithm}[t]
    \SetAlgoLined
    \KwIn{Objective $f:\mathcal{X} \rightarrow \mathbb{R}$;
    Acquisition criterion $\alpha$; Original dimension $D$; Embedding dimension $d$; Initial points $t_0$;
    Evaluation trials $t_N$}
    \KwOut{Best point $\bsx \in \underset{\mathcal{X}}{\arg \max} \; f(\bsx)$}
    Uniformly sample $t_0$ points $\{\bsx_1, \bsx_2, \cdots, \bsx_{t_0}\}$ in the original space;

    Define $\mathcal{D}_0=\{(\bsx_1,f(\bsx_1)),(\bsx_2,f(\bsx_2)),\ldots,(\bsx_{t_0},f(\bsx_{t_0}))\}$;

    \While{$t_0 + 1 \leq t \leq t_N$}{
        Construct the projection matrix $\mathbf{A}_t \in \mathbb{R}^{d \times D}$ according to Gaussian projection matrix in the Definition \ref{random Gaussian matrix definition} or hashing projection matrix in the Definition \ref{hashing random projection matrix};
        
        Project the points in $\mathcal{D}_{t-1}$ onto the embedding subspace $\mathcal{Y}_t$ via $\mathbf{A}_t$, 
        obtaining the set of points in the embedding subspace $\mathcal{D}_{t-1}^y=\{(\mathbf{A}_t\bsx_1,f(\bsx_1)), (\mathbf{A}_t\bsx_2,f(\bsx_2)), \cdots, (\mathbf{A}_t\bsx_{t-1},f(\bsx_{t-1}))\}$\;
        
        Estimate the hyperparameters $\theta_t$ of the Gaussian Process prior for the given $\mathcal{D}_{t-1}^y$;
        
        Calculate the posterior probability of the Gaussian Process based on $\mathcal{D}_{t-1}^y$ and the estimated hyperparameters $\theta_t$.

        Compute the maximum of the acquisition criterion $\alpha$, $\mbf{y_t} \in \underset{\mbf{y}\in \mathcal{Y}}{\arg \max} \; \alpha(\mbf{y} \mid \mathcal{D}_{t-1}^y)$;

        Project $\mbf{y}_t$ back to the original space via $\mathbf{A}^\top_t$, obtaining $\bsx_t=\mathbf{A}^\top_t \mbf{y_t}$;
        
        Update the dataset $\mathcal{D}_{t} = \mathcal{D}_{t-1} \cup \{(\bsx_t,f(\bsx_t))\}$, and $t = t + 1$.
    }
    \caption{CEPBO Algorithm}
    \label{CEPBO Algorithm}
\end{algorithm}

\subsection{Address the boundary issue}
Our approach, akin to REMBO, encounters the issue of excessive exploration along the boundary of $\mathcal{X}$. 
To ensure the effective tuning of the acquisition function and to facilitate BO, 
it is crucial for the embedding subspace $\mathcal{Y}$ to have a bounded domain.
However, random projections between the original space of dimension $D$ and the embedding subspace of dimension $d$ can lead to points exceeding domain boundaries after CEP. These exceedances occur in two scenarios: $\mbf{y} = \mathbf{A} \bsx \notin \mathcal{Y}, \; \tilde{\bsx} = \mathbf{A}^\top \mbf{y} \notin \mathcal{X}.$ 

To mitigate the issue, we employ the convex projection of the original space, $P_{\mathcal{X}}$, and that of the 
embedding subspace, $P_{\mathcal{Y}}$, to rectify boundary transgressions. 
Specifically, the convex projection within the original space $\mathcal{X}$ is defined as follows: 
$$P_{\mathcal{X}}: \mathbb{R}^D \to \mathbb{R}^D,  \; P_{\mathcal{X}}(\tilde{\bsx})=\underset{\mbf{z} \in \mathcal{X}}{\arg \min} \|\mbf{z} -\tilde{\bsx}\| _2 .$$ 
Similarly, the convex projection within the embedding subspace $\mathcal{Y}$ is expressed as:  
$$P_{\mathcal{Y}}: \mathbb{R}^d \to \mathbb{R}^d,  \; P_{\mathcal{Y}}(\mbf{y})=\underset{\mbf{z} \in \mathcal{Y}}{\arg \min} \|\mbf{z} -\mbf{y}\| _2 .$$

Convex projection will lead to an issue where multiple distinctive values in the original space are mapped to identical boundary points within the embedding subspace, i.e.,   
for $\bsx_1 \neq \bsx_2$ such that $f(\bsx_1) \neq f(\bsx_2)$, the equality $P_{\mathcal{Y}}(\mathbf{A}\bsx_1) = P_{\mathcal{Y}}(\mathbf{A}\bsx_2)$ holds.  
Moreover, the substantial disparity between the dimensions $d$ and $D$ exacerbates the likelihood of such instances. This issue can undermine the precision of Gaussian process models and consequently, diminish the efficacy of optimization. 
To mitigate this risk, a scaling strategy is implemented within the Condensing Projection and Expansion Projection phases to diminish the probability of such occurrences. 
This involves scaling the projected points $\mathbf{A}\bsx$ using a reduction factor before applying convex projection, as follows: 
$$\mbf{y} = P_{\mathcal{Y}}\left(\frac{1}{\sqrt{D}}\mathbf{A}\bsx\right).$$ 
In a parallel procedure, the optimal points of the acquisition function in the embedding subspace $\mbf{y}$ undergo an inverse scaling: 
$$\tilde{\bsx} = P_{\mathcal{X}}(\sqrt{D}\mathbf{A}^\top \mbf{y}).$$ 
In this context, the scaling factors $\frac{1}{\sqrt{D}}$ and $\sqrt{D}$ confine the scope of projection within the viable domain. These factors are verified through empirical experimentation. 

\section{Experiments}
We conduct experiments to demonstrate the performance of the proposed method across various functions and 
real-world scenarios. 
In Section \ref{sec:num}, we evaluate its performance on three benchmark functions. 
In Section \ref{sec:real}, we assess it  across four real-world problems.  
These experimental results indicate that our algorithms, CEP-REMBO and CEP-HeSBO, achieve superior results.

Because our approach, CEPBO, represents an advancement in the domain of the linear embeddings, our experiments focus on comparing it with other linear embeddings.  
We aim to assess the improvement achieved by applying CEP compared to REMBO\cite{rembo} and HeSBO\cite{hesbo}, respectively. ALEBO is considered to achieve state-of-the-art performance on this class
of optimization problems with a true linear subspace.
Therefore, we chose REMBO\cite{rembo}, HeSBO\cite{hesbo}, and ALEBO\cite{alebo} as the benchmark algorithms for our comparisons.

\subsection{Numerical Results}
\label{sec:num}
We evaluated the performance of the algorithms using the following benchmark functions: 
(1) the Holder Table function, (2) the Schwefel function, and (3) the Griewank function. 
Each function's input space was extended to a dimensionality of $D=100$. The Holder Table function has an effective dimension of $2$, whereas the Schwefel and Griewank functions are fully defined over the entire $D$-dimensional space. 
To enhance the complexity of the optimization task, we introduced random perturbations into both the Schwefel and Griewank functions; for further details, please refer to the appendix \ref{Modified Schwefel function and Griewank function}.
The goal is to find the minimum value of these functions. 
The number of initialization trials
for each algorithm was kept the same as the dimensionality of its embedding subspace. Each experiment is independently repeated 50 
times, with 50 evaluations per experiment. 
To assess the performance of the CEPBO algorithm under various embedding space dimensions, we take $d=2, 5, 10$ in the Holder Table function and $d=2, 5, 20$ in the Schwefel and Griewank functions. Since an effective dimension for the Schwefel and Griewank functions is 100, we prioritize the larger $d$ for assessment. In these experiments, we utilize expected improvement as the acquisition function. 
We have also conducted an experiment using the well-known Hartmann function with $D=1000$ as reported in the appendix \eqref{app:d1000}.  The results shows that the CEP projection can significantly improve the performance of both REMBO and HeSBO.

\begin{figure*}[ht]
    \centering
    \includegraphics[width=1\textwidth]{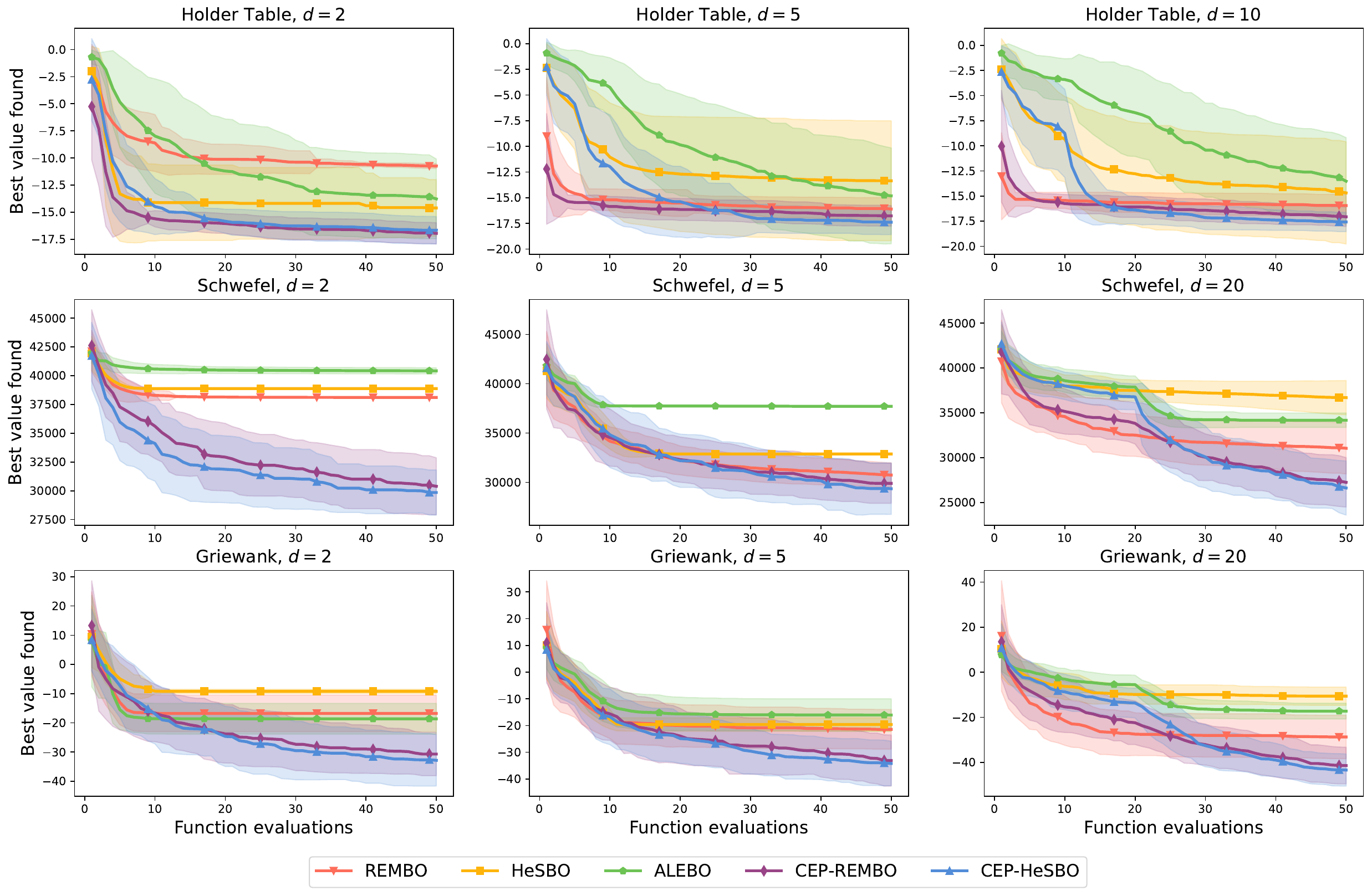}
    \caption{The results of the optimization experiments for three functions across various embedding dimensions. From top to bottom: Holder Table, Schwefel, and Griewank functions. 
    }
    \label{fig:ax}
\end{figure*}

We report the results in Figure \ref{fig:ax}. 
For the Schwefel and Griewank functions, where embedding dimensions are smaller than the effective dimension of 100, the baseline algorithms nearly ceased functioning, settling in local optima, which is visually depicted as a flat horizontal line on the corresponding graphs. 
Interestingly, even in the instance of the Holder Table function, where 
the embedding dimension met or exceeded the effective dimension, a circumstance where the baseline algorithms typically perform well, the approached algorithms continued to show superior performance over all baselines. 
Comparative the performance across a range of embedding dimensions $d$, the performance are similar and CEP-REMBO and CEP-HeSBO consistently surpassed all baseline algorithms. Therefore, the REMBO and HeSBO algorithms experienced substantial improvement with the integration  of the Condense-Expansion Projection mechanism.

\subsection{Real-World Problems}
\label{sec:real}
In this section, we evaluate the CEPBO algorithm on real-world optimization problems. 
The test cases consist of lunar landing task in the realm of reinforcement learning with $D=12$ \cite{eriksson2019scalable}, a robot pushing problem with $D=14$ \cite{wang2017batched}, a problem in neural architecture search with $D=36$ \cite{letham2020re}, and a rover trajectory planning problem with $D=60$ \cite{wang2018batched}. 
Algorithmic configurations and acquisition function selections strictly adhere to the  settings outlined in the original papers.
For additional details, please refer to the appendix.
The optimization goal is to maximize the reward function, and each experiment is independently repeated 10 times, with 500 evaluations per experiment.

\begin{figure*}[ht]
    \centering
    \includegraphics[width=1\textwidth]{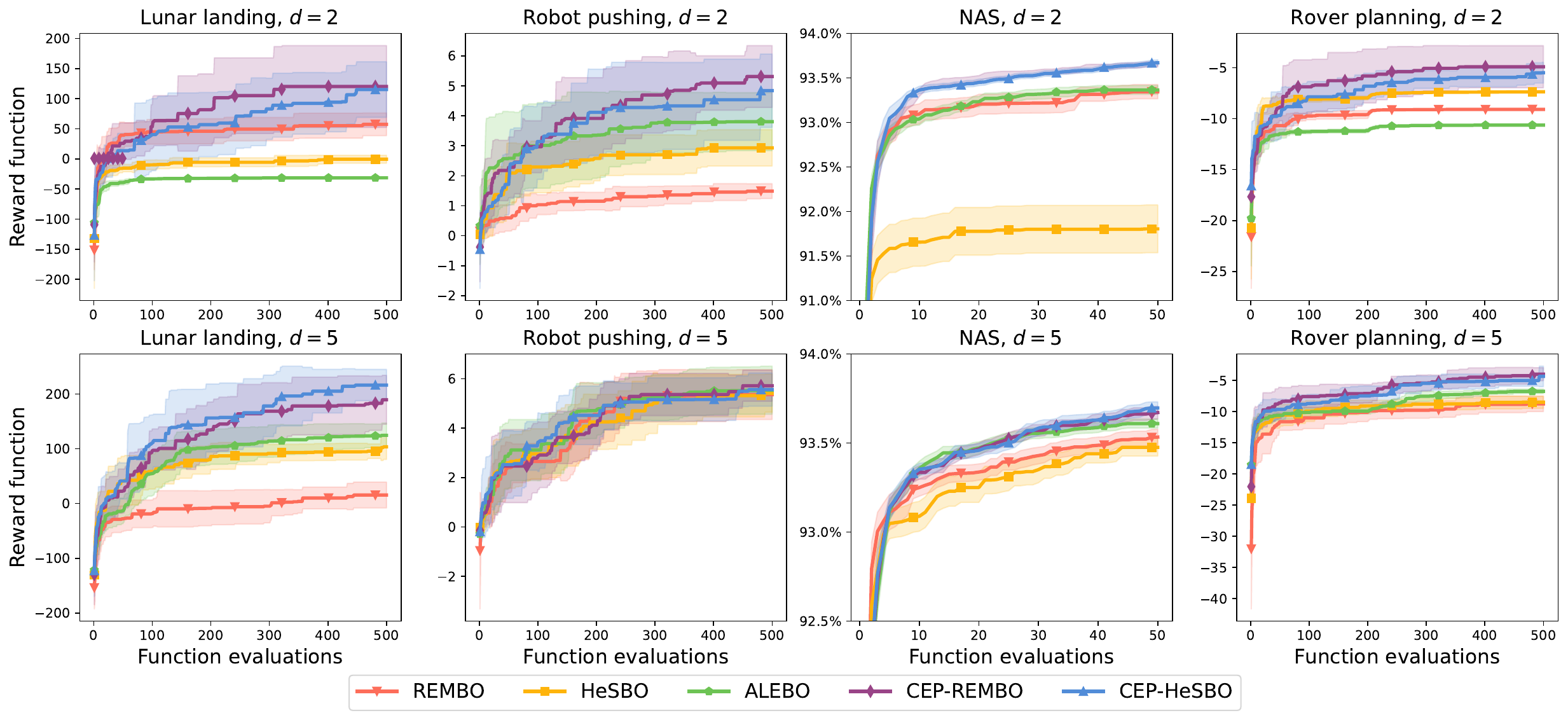}
    \caption{The results of the optimization experiments for four real-world scenarios.
    From left to right: Lunar landing, Robot pushing, NAS, and Rover planning. 
    }
    \label{fig:real}
\end{figure*}

\textbf{Lunar Landing.}\; This experiment entails the task of devising a reinforcement 
learning strategy for the lunar lander's control mechanism, aiming to minimize fuel 
consumption and proximity to the landing site while preventing a crash. The original space dimension is $D=12$. 
In the first column of Figure \ref{fig:real}, REMBO, HeSBO, and ALEBO algorithms become trapped in local optima to different extents. As the dimensionality of the embedding subspace $d$ increases from $d=2$ to $d=5$, notable performance improvements are observed for most algorithms, except REMBO. 
The introduced CEP-REMBO and CEP-HeSBO approaches consistently demonstrate the ability to enhance and identify novel optimal resolutions.

\textbf{Robot Pushing.}\; This scenario involves a robotics dual-arm manipulation 
task where the robot's arms are controlled by adjusting 14 modifiable parameters to push 
two objects while tracking their movement trajectories. 
The original space dimension is $D=14$. 
The second column of Figure \ref{fig:real} 
demonstrates that the proposed CEP-REMBO and CEP-HeSBO methods significantly 
outperform others when $d = 2$. When increasing to $d = 5$, all methods exhibit varying degrees of performance enhancement. This suggests that for optimization issues with moderate to low dimensionalities, escalating the dimensions of the embedding subspaces can notably bolster the algorithms' efficacy. Notwithstanding these improvements, the CEP-REMBO and CEP-HeSBO methods consistently maintain their leading positions.

\textbf{Neural Architecture Search (NAS).}\; The objective of this experiment is to identify an optimal 
architecture for neural networks, paralleling the methodology utilized by \cite{alebo}. Leveraging data 
from the NAS-Bench-101 benchmark \cite{ying2019bench}, we have developed an optimization problem focused on searching for a 
convolutional neural network architecture characterized by 36 dimensions. 
The original space dimension is $D=36$. 
In the third column of Figure \ref{fig:real}, at an embedding 
subspace dimension of $d=2$, the REMBO, HeSBO, and ALEBO algorithms rapidly converge to less than ideal solutions, hindering the exploration of superior neural network structures. 
On the contrary, the CEP-REMBO and CEP-HeSBO methods maintain the capability to persistently optimize, discovering architectures 
with improved accuracy. Increasing the subspace dimension to $d=5$ reveals the ALEBO's enhanced ability 
to perform on par with CEP-REMBO and CEP-HeSBO methods; however, CEP-HeSBO consistently exhibits the 
highest performance across all conditions.

\textbf{Rover Trajectory Planning.}\; This task involves optimizing a 2D
trajectory comprising of 30 pivotal points that collectively define a navigational path. 
The original space dimension is $D=60$. 
The fourth column of Figure \ref{fig:real} indicates that for an embedding 
subspace dimension of $d = 2$, the REMBO, HeSBO, and ALEBO algorithms can not successfully converge to an advantageous reward. In contrast, the CEP-REMBO and CEPHeSBO algorithms exhibit a capacity to consistently identify superior solutions. 
This pattern is similarly observed when the subspace dimension is increased to $d = 5$.

\section{Conclusion}
This paper proposes a Bayesian optimization framework utilizing the Condensing-Expansion Projection technique, free from reliance on the assumption of effective dimension. 
The primary concept involves employing projection twice within each iteration: first, projecting to an embedding subspace, and then projecting back to retain optimization information in the original high-dimensional space. 
The approach does not impose additional requirements on the projection matrix used, thereby significantly enhancing the applicability of the embedding-based Bayesian optimization algorithms. This flexibility enables the selection of a suitable projection matrix according to the problem's characteristics beforehand. 
This paper also theoretically verifies the properties of Condensing-Expansion Projection 
under random Gaussian matrix and hashing matrix, providing that the transformed point $\tilde{\bsx}$ concentrates around the point $\bsx$ in the original space. 
Finally, two new Bayesian optimization algorithms based on Condensing-Expansion Projection are proposed: CEP-REMBO and CEP-HeSBO based on the Gaussian projection matrix and the hash-enhanced projection matrix, respectively.   
Empirically, this paper conducts comprehensive experiments to assess the performance of the proposed algorithms across diverse optimization scenarios. 
The experimental results demonstrate that the Bayesian optimization algorithms based on Condensing-Expansion Projection achieved promising performance across these optimization functions, overcoming the reliance on effective dimension.

Our work has several limitations that can be addressed by future works. 
For instance, while our analysis supports the concentration of the transformed point around the original point after employing CEP, we acknowledge the absence of $\epsilon$-subspace embedding to preserve the mean and variance functions of a Gaussian process. Although reference to \cite{hesbo} suggests a potential avenue for such analysis, a more thorough examination is warranted. Another constraint of our study is the absence of an evaluation of CEP-based algorithms for optimization problems with billions of dimensions. Despite this potential, our approach lacks empirical validation, whereas REMBO has been shown to effectively address such challenges. 
\newpage

\bibliographystyle{unsrtnat}
\bibliography{main}

\newpage
\section{Appendix}
\label{Appendix}

\subsection{Proof of Theorem \ref{expectation_properties}}
\label{Proof of Property expectation_properties}
\begin{proof}
(1) First, we consider the Gaussian Projection. Let $\mbf{A}=(\sbf{l}_1,\cdots,\sbf{l}_d)^\top$, such that $\sbf{l}_i=(\alpha_{i1},\cdots
,\alpha_{iD})^\top$ is a $D\times 1$ vector where each element is from $\mathcal{N}\left(0,\frac{1}{d}\right)$ distribution.
It is easy to show that, given any vector $\mbf{x}\in\mathcal{X}$ and any vector $\mbf{g}\in\mathbb{R}^D$,
\begin{equation}\label{A-unbiased}
    \mathbb{E}(\mbf{A}^\top\mbf{A})=\sum\limits_{i=1}^d\mathbb{E}\left(\sbf{l}_i\sbf{l}_i^\top\right)=\mbf{I}.
\end{equation}
(2) Second, we consider the Hashing Random Projection. 
For the Hashing Random Matrix, we rewrite 
 $\mbf{A} = \mbf{S} \mbf{D}$, where $\mbf{S}\in
\mathbb{R}^{d\times D}$ has each column chosen independently and uniformly from 
the $r$ standard basis vectors of $\mathbb{R}^d$ and $\mbf{D}\in \mathbb{R}^{D\times D}$
is a diagonal matrix with diagonal entries chosen independently and uniformly  on $b_i\in\{\pm 1\}$.

Let $\mbf{S}=(\sbf{s}_1,\cdots,\sbf{s}_D)$, such that $\sbf{s}_i$ is a random vector 
taking the vector $e_j$ 
for equal probability, where $e_k$ is the $k$th standard unit vector of $\mathbb{R}^{d}$ 
for $k=1,\cdots,d$. 
Then $\mathbb{E}(\sbf{s}_i)=\frac{1}{d}\sbf{1}_d$ and $\mathbb{E}(\sbf{s}_i^\top\sbf{s}_i)=1$, which follow that
$\mathbb{E}[(\mbf{S}^\top\mbf{S})_{ij}]=\mathbb{E}(\sbf{s}_i^\top\sbf{s}_j)=\mathbb{E}(\sbf{s}_i)^\top \mathbb{E}(\sbf{s}_j)=\frac{1}{d^2}$
for $i\neq j$;
and $\mathbb{E}[(\mbf{S}^\top\mbf{S})_{ii}]=\mathbb{E}(\sbf{s}_i^\top\sbf{s}_i)=1$.
Obviously, $$\mathbb{E}(\mbf{A}^\top\mbf{A})=\mbf{I}.$$

\end{proof}

\subsection{Proof of Theorem \ref{prop:Variance}}
\label{Proof of Property Variance}

\begin{proof}
(1) First, we consider the Gaussian projection. Let $\mbf{A}=(\sbf{l}_1,\cdots,\sbf{l}_d)^\top$, such that $\sbf{l}_i=(\alpha_{i1},\cdots
,\alpha_{iD})^\top$ is a $D\times 1$ vector where each element is from $\mathcal{N}\left(0,\frac{1}{d}\right)$ distribution.
\eqref{A-unbiased} follows that 
\begin{align}\label{moment-proof}
   \mathbb{E}[(\mbf{g}^\top\mbf{A}^\top\mbf{A}\mbf{x}
   -\mbf{g}^\top\mbf{x})^2]
   =& \mathbb{E}[(\mbf{g}^\top\mbf{A}^\top\mbf{A}\mbf{x})^2]
   -(\mbf{g}^\top\mbf{x})^2.
\end{align}
In \eqref{moment-proof}, we have,
\begin{align}\label{E2-A}
&\mathbb{E}\left[(\mbf{g}^\top\mbf{A}^\top\mbf{A}\mbf{x})^2\right]=\mathbb{E}\Big[(\sum\limits_{i=1}^d\mbf{g}^\top\sbf{l}_i\sbf{l}_i^\top\mbf{x})^2\Big]\notag\\
=&\mathbb{E}\Big[\sum\limits_{i=1}^d
(\mbf{g}^\top\sbf{l}_i\sbf{l}_i^\top\mbf{x}\mbf{x}^{\top}\sbf{l}_i\sbf{l}_i^\top\mbf{g}
)+\sum\limits_{i\neq j}(\mbf{g}^\top\sbf{l}_i\sbf{l}_i^\top\mbf{x}
)(\mbf{g}^\top\sbf{l}_j\sbf{l}_j^\top\mbf{x})^\top\Big]\\\notag
=&d \mathbb{E}[(\mbf{g}^\top\sbf{l}_i\sbf{l}_i^\top\mbf{x})(\mbf{x}^\top\sbf{l}_i\sbf{l}_i^\top\mbf{g})]+(d^2-d
)\mathbb{E}(\mbf{g}^\top\sbf{l}_i\sbf{l}_i^\top\mbf{x})\mathbb{E}(\mbf{g}^\top\sbf{l}_j\sbf{l}_j^\top\mbf{x})^\top.
\end{align}
By (\ref{E2-11-GP}) in Lemma \ref{lemma-gp}, 
$$\mathbb{E}[(\mbf{g}^\top\sbf{l}_i\sbf{l}_i^\top\mbf{x})(\mbf{x}^\top\sbf{l}_i\sbf{l}_i^\top\mbf{g})]
=\frac{3}{d^2}\|\mbf{x}\|^2\|\mbf{g}\|^2.
$$

Then we have 
\begin{equation}\label{E2-M-result}
\mathbb{E}\left[(\mbf{g}^\top\mbf{A}^\top\mbf{A}\mbf{x})^2\right]
=(\mbf{g}\top\mbf{x})^2+\frac{3}{d}\|\mbf{x}\|^2\|\mbf{g}\|^2-\frac{1}{d}(\mbf{g}\top\mbf{x})^2.
\end{equation}
Therefore, we have
\begin{equation*}
\mathbb{E}\left[(\mbf{g}^\top\mbf{A}^\top\mbf{A}\mbf{x}-\mbf{g}^\top\mbf{x}
)^2\right]=
\frac{3}{d}\|\mbf{x}\|^2\|\mbf{g}\|^2-\frac{1}{d}(\mbf{g}\top\mbf{x})^2.
\end{equation*}



(2) Second, we consider the Hashing Random Projection. 
For the Hashing Random Matrix, we rewrite 
 $\mbf{A} = \mbf{S} \mbf{D}$, where $\mbf{S}\in
\mathbb{R}^{d\times D}$ has each column chosen independently and uniformly from 
the $r$ standard basis vectors of $\mathbb{R}^d$ and $\mbf{D}\in \mathbb{R}^{D\times D}$
is a diagonal matrix with diagonal entries chosen independently and uniformly  on $b_i\in\{\pm 1\}$.

Let $\mbf{S}=(\sbf{s}_1,\cdots,\sbf{s}_D)$, such that $\sbf{s}_i$ is a random vector 
taking the vector $e_j$ 
for equal probability, where $e_k$ is the $k$th standard unit vector of $\mathbb{R}^{d}$ 
for $k=1,\cdots,d$. 
Then $\mathbb{E}(\sbf{s}_i)=\frac{1}{d}\sbf{1}_d$ and $\mathbb{E}(\sbf{s}_i^\top\sbf{s}_i)=1$, which follow that
$\mathbb{E}[(\mbf{S}^\top\mbf{S})_{ij}]=\mathbb{E}(\sbf{s}_i^\top\sbf{s}_j)=\mathbb{E}(\sbf{s}_i)^\top \mathbb{E}(\sbf{s}_j)=\frac{1}{d^2}$
for $i\neq j$;
and $\mathbb{E}[(\mbf{S}^\top\mbf{S})_{ii}]=\mathbb{E}(\sbf{s}_i^\top\sbf{s}_i)=1$.

Applying Lemma \ref{lemma-CW}, 
\begin{align*}
    \mathbb{E}\left[(\mbf{g}^\top\mbf{A}^\top\mbf{A}\mbf{x})^2\right]
    =&(\mbf{g}^\top\mbf{x})^2+\frac{1}{d}\left(\|\mbf{g}\|^2\|\mbf{x}\|^2-\sum\limits_{i=1}^Dx_i^2g_i^2\right).
\end{align*}
Thus, we have, 
$$\mathbb{E}\left[(\mbf{g}^\top\mbf{A}^\top\mbf{A}\mbf{x}-\mbf{g}^\top\mbf{x})^2\right]
 = \frac{1}{d}\left(\|\mbf{g}\|^2\|\mbf{x}\|^2-\sum\limits_{i=1}^Dx_i^2g_i^2\right).
$$


\end{proof}

\subsubsection{Two Lemmas}
\begin{lemma}\label{lemma-gp}
Let $\mbf{A}=(\sbf{l}_1,\cdots,\sbf{l}_d)^{\top}$, such that $\sbf{l}_i=(\alpha_{i1},
\cdots,\alpha_{iD})^{\top}$ is an $D\times1$ vector where each element is from zero mean
distribution with $E(\alpha_{ij}^2) = 1$ and $E(\alpha_{ij}^4)=\gamma$, we have that, for any matrices $\mbf{M}_1\in\mathbb{R}^{D\times d_1}$ and $\mbf{M}_2\in\mathbb{R}^{D\times d_2}$, where $\sbf{m}_{1i}$ and $\sbf{m}_{2i}$ are their $i$-th row, respectively. 

\begin{equation}\label{E2-11}
E\left[(\mbf{M}_1^{\top}\sbf{l}_{i}\sbf{l}_i^{\top}\mbf{M}_2
)(\mbf{M}_2^{\top}\sbf{l}_i\sbf{l}_i^{\top}\mbf{M}_1)\right]=\mbf{M}_1^{\top}[(\gamma-3)\mbf{W}_2 
+ 2\mbf{M}_2\mbf{M}_2^{\top}+ \text{tr}(\mbf{M}_{2}\mbf{M}_{2}^{T})\mbf{I}]\mbf{M}_1,
\end{equation}
where $\mbf{W}_2=\text{diag}\{\sbf{m}_{21}^{\top}\sbf{m}_{21},\cdots,\sbf{m}_{2D}^{\top}\sbf{m}_{2D}\}$.
Particularly, for Gaussian projection,

\begin{equation}\label{E2-11-GP}
E\left[(\mbf{M}_1^{\top}\sbf{l}_{i}\sbf{l}_i^{\top}\mbf{M}_2
)(\mbf{M}_2^{\top}\sbf{l}_i\sbf{l}_i^{\top}\mbf{M}_1)\right]=2\mbf{M}_1^{\top}[\mbf{M}_2\mbf{M}_2^{\top}]\mbf{M}_1 
+ \text{tr}(\mbf{M}_{2}\mbf{M}_{2}^{T})\mbf{M}_{1}^{T}\mbf{M}_{1},
\end{equation}

\end{lemma}

\begin{proof}
Since $\mbf{M}_1^{\top}\sbf{l}_i=\sum\limits_{j=1}^D\alpha_{ij}\sbf{m}_{1j}$ and 
$\mbf{M}_2^{\top}\sbf{l}_i=\sum\limits_{j=1}^D\alpha_{ij}\sbf{m}_{2j}$, where $\sbf{m}_{1j}^{\top}$
and $\sbf{m}_{2j}^{\top}$ are the $j$th row of $\mbf{M}_1$ and $\mbf{M}_2$ respectively, it 
follows that,
$$\mbf{M}_1^{\top}\sbf{l}_{i}\sbf{l}_i^{\top}\mbf{M}_2=(\sum\limits_{j=1}^D\alpha_{ij}\sbf{m}_{1j}
)(\sum\limits_{j=1}^D\alpha_{ij}\sbf{m}_{2j})^{\top}=\sum\limits_{j
=1}^n\alpha_{ij}^2\sbf{m}_{1j}\sbf{m}_{2j}^{\top}+\sum\limits_{j_1\neq 
j_2}\alpha_{ij_1}\alpha_{ij_2}\sbf{m}_{1j_1}\sbf{m}_{2j_2}^{\top}.$$
$$\mbf{M}_2^{\top}\sbf{l}_i\sbf{l}_i^{\top}\mbf{M}_1=(\sum\limits_{j=1}^D\alpha_{ij}\sbf{m}_{2j}
)(\sum\limits_{j=1}^D\alpha_{ij}\sbf{m}_{1j})^{\top}=\sum\limits_{j
=1}^n\alpha_{ij}^2\sbf{m}_{2j}\sbf{m}_{1j}^{\top}+\sum\limits_{j_1\neq 
j_2}\alpha_{ij_1}\alpha_{ij_2}\sbf{m}_{2j_1}\sbf{m}_{1j_2}^{\top},$$

Since $E(\alpha_{ij}^4)=\gamma$, so, we have that
\begin{align*}
&E\left(\sum\limits_{j=1}^D\alpha_{ij}^2\sbf{m}_{1j}\sbf{m}_{2j}^{\top}\right)\left(\sum\limits_{j
=1}^n\alpha_{ij}^2\sbf{m}_{2j}\sbf{m}_{1j}^{\top}\right)\notag\\
=&\gamma\sum\limits_{j=1}^D\sbf{m}_{1j}\sbf{m}_{2j}^{\top}\sbf{m}_{2j}\sbf{m}_{1j}^{\top}+\left
(\sum\limits_{j_1\neq j_2}\sbf{m}_{1j_1}\sbf{m}_{2j_1}^{\top}\sbf{m}_{2j_2}\sbf{m}_{1j_2}^{\top}\right),
\end{align*}
\begin{align*}
&E\left(\sum\limits_{j_1\neq j_2}\alpha_{ij_1}\alpha_{ij_2}\sbf{m}_{1j_1}\sbf{m}_{2j_2}^{\top}\right
)\left(\sum\limits_{j_1\neq j_2}\alpha_{ij_1}\alpha_{ij_2}\sbf{m}_{2j_1}\sbf{m}_{1j_2}^{\top}\right)\notag\\
=&\left(\sum\limits_{j_1\neq j_2}\sbf{m}_{1j_1}\sbf{m}_{2j_2}^{\top}\sbf{m}_{2j_1}\sbf{m}_{1j_2}^{\top}
\right) + \sum\limits_{j_1\neq j_2}\sbf{m}_{1j_{1}}\sbf{m}_{2j_{2}}^{T}\sbf{m}_{2j_{2}}
\sbf{m}_{1j_{1}}^{T},
\end{align*}
$$
E\left(\sum\limits_{j=1}^D\alpha_{ij}^2\sbf{m}_{1j}\sbf{m}_{2j}^{\top}\right)\left
(\sum\limits_{j_1\neq j_2}\alpha_{ij_1}\alpha_{ij_2}\sbf{m}_{2j_1}\sbf{m}_{1j_2}^{\top}\right)=0,
$$
$$
E\left(\sum\limits_{j_1\neq j_2}\alpha_{ij_1}\alpha_{ij_2}\sbf{m}_{1j_1}\sbf{m}_{2j_2}^{\top}
\right)\left(\sum\limits_{j=1}^D\alpha_{ij}^2\sbf{m}_{2j}\sbf{m}_{1j}^{\top}\right)=0.
$$

Combing the four equations above, it is easy to verify that,
\begin{align*}
&E\left[(\mbf{M}_1^{\top}\sbf{l}_{i}\sbf{l}_i^{\top}\mbf{M}_2
)\mbf{M}_2^{\top}\sbf{l}_i\sbf{l}_i^{\top}\mbf{M}_1\right]\notag\\
=&(\gamma-3)\sum\limits_{j
=1}^n\sbf{m}_{1j}\sbf{m}_{2j}^{\top}\sbf{m}_{2j}\sbf{m}_{1j}^{\top} + 
2\mbf{M}_1^{\top}\mbf{M}_2\mbf{M}_2^{\top}\mbf{M}_1 + \text{tr}(\mbf{M}_{2}\mbf{M}_{2}^{T}
)\mbf{M}_{1}^{T}\mbf{M}_{1}.
\end{align*}

\end{proof}


\begin{lemma}\label{lemma-CW}
For the Hashing random projection
\begin{equation*}
E\left[(\mbf{M}_1^{\top}\mbf{A}^{\top}\mbf{A}\mbf{M}_2)
(\mbf{M}_1^{\top}\mbf{A}^{\top}\mbf{A}
\mbf{M}_2)^{\top}\right]=\mbf{M}_1^{\top}\mbf{M}_2\mbf{M}_2^{\top}\mbf{M}_1 + \frac{1}{d}
\mbf{M}_1^{\top}\mbf{M}_1
\text{tr}(\mbf{M}_2\mbf{M}_2^{\top})- \frac{1}{d}\mbf{M}_{1}^{\top}\mbf{W}_2\mbf{M}_{1}.
\end{equation*}
\end{lemma}

\begin{proof}
For the Hashing Random Matrix projection, we rewrite 
 $\mbf{A} = \mbf{S} \mbf{D}$, where $\mbf{S}\in
\mathbb{R}^{d\times D}$ has each column chosen independently and uniformly from 
the $r$ standard basis vectors of $\mathbb{R}^d$ and $\mbf{D}\in \mathbb{R}^{D\times D}$
is a diagonal matrix with diagonal entries chosen independently and uniformly  on $b_i\in\{\pm 1\}$.

Let $\mbf{S}=(\sbf{s}_1,\cdots,\sbf{s}_n)$, such that $\sbf{s}_i$ is a random vector 
taking the vector $e_j$ 
for equal probability, where $e_j$ is the $j$th standard unit vector of $\mathbb{R}^{d}$ 
for $j=1,\cdots,r$. 
Then $E(\sbf{s}_i)=\frac{1}{d}\sbf{1}_d$ and $E(\sbf{s}_i^{\top}\sbf{s}_i)=1$, which follow that
$E[(\mbf{S}^{\top}\mbf{S})_{ij}]=E(\sbf{s}_i^{\top}\sbf{s}_j)=E(\sbf{s}_i)^{\top}E(\sbf{s}_j)=\frac{1}{d^2}$
for $i\neq j$;
and $E[(\mbf{S}^{\top}\mbf{S})_{ii}]=E(\sbf{s}_i^{\top}\sbf{s}_i)=1$.
We have that,
\begin{align}\label{A12}
\mbf{M}_1^{\top}\mbf{A}^{\top}\mbf{A}\mbf{M}_2&=(\sum\limits_{i=1}^D
b_i\sbf{s}_i\sbf{m}_{1i}^{\top})^{\top}(\sum
\limits_{i=1}^D b_i\sbf{s}_i\sbf{m}_{2i}^{\top})\notag\\
&=\sum\limits_{i=1}^D b_i^2\sbf{m}_{1i}\sbf{s}_i^{\top}\sbf{s}_i\sbf{m}_{2i}^{\top}+\sum\limits
_{i\neq j}
b_ib_j\sbf{m}_{1i}\sbf{s}_i^{\top}\sbf{s}_j\sbf{m}_{2j}^{\top}
\end{align}

From \eqref{A12}, we have,
\begin{align*}
    &E\left[\mbf{M}_1^{\top}\mbf{A}^{\top}\mbf{A}\mbf{M}_2\right]\left[\mbf{M}_1^{\top}
    \mbf{A}^{\top}
    \mbf{A}\mbf{M}_2\right]^{\top}\notag\\
    =&E\left(\sum\limits_{i=1}^D b_i^2\sbf{m}_{1i}\sbf{s}_i^{\top}\sbf{s}_i\sbf{m}_{2i}^{\top}+\sum
    \limits_{i\neq j} b_ib_j\sbf{m}_{1i}\sbf{s}_i^{\top}\sbf{s}_j\sbf{m}_{2j}^{\top}\right)\left(\sum
    \limits_{i=1}^D b_i^2\sbf{m}_{1i}\sbf{s}_i^{\top}\sbf{s}_i\sbf{m}_{2i}^{\top}+\sum\limits_{i\neq j}
    b_ib_j\sbf{m}_{1i}\sbf{s}_i^{\top}\sbf{s}_j\sbf{m}_{2j}^{\top}\right)^{\top}\notag\\
   =&\sum\limits_{i=1}^D E(\sbf{m}_{1i}\sbf{s}_i^{\top}\sbf{s}_i\sbf{m}_{2i}
   ^{\top})(\sbf{m}_{1i}\sbf{s}_i^{\top}\sbf{s}_i\sbf{m}_{2i}^{\top})^{\top}+\sum\limits_{i\neq j}
   E(\sbf{m}_{1i}\sbf{s}_i^{\top}\sbf{s}_i\sbf{m}_{2i}^{\top})
   (\sbf{m}_{1j}\sbf{s}_j^{\top}\sbf{s}_j\sbf{m}_{2j}^{\top})^{\top}\\\notag
 & +\sum\limits_{i\neq j} E(\sbf{m}_{1i}\sbf{s}_i^{\top}\sbf{s}_j\sbf{m}_{2j}^{\top})
 (\sbf{m}_{1i}\sbf{s}_i^{\top}\sbf{s}_j\sbf{m}_{2j}^{\top})^{\top}
=:E_1+E_2+E_3.
\end{align*}
Specifically, we have that

\begin{align}\label{E}
&E_1=\sum\limits_{i=1}^D E(\sbf{m}_{1i}\sbf{s}_i^{\top}\sbf{s}_i\sbf{m}_{2i}^{\top})
(\sbf{m}_{1i}\sbf{s}_i^{\top}\sbf{s}_i\sbf{m}_{2i}^{\top})^{\top}=\sum\limits_{i=1}^D \sum\limits_{k=1}^d
\frac{1}{d}
(\sbf{m}_{1i}\mbf{e}_k^{\top}\mbf{e}_k\sbf{m}_{2i}^{\top})
(\sbf{m}_{1i}\mbf{e}_k^{\top}\mbf{e}_k\sbf{m}_{2i}^{\top})^{\top}\notag\\
&\ \ =\sum\limits_{i=1}^D (\sbf{m}_{1i}\sbf{m}_{2i}^{\top})(\sbf{m}_{1i}\sbf{m}_{2i}^{\top})^{\top}.\notag\\
&E_2=\sum\limits_{i\neq j} E(\sbf{m}_{1i}\sbf{s}_i^{\top}\sbf{s}_i\sbf{m}_{2i}^{\top})
(\sbf{m}_{1j}\sbf{s}_j^{\top}\sbf{s}_j\sbf{m}_{2j}^{\top})^{\top}=\sum\limits_{i\neq j} E
(\sbf{m}_{1i}\sbf{s}_i^{\top}\sbf{s}_i\sbf{m}_{2i}^{\top})
E(\sbf{m}_{1j}\sbf{s}_j^{\top}\sbf{s}_j\sbf{m}_{2j}^{\top})^{\top}\notag\\
&\ \ =\sum\limits_{i\neq j} (\sbf{m}_{1i}\sbf{m}_{2i}^{\top})(\sbf{m}_{1j}\sbf{m}_{2j}^{\top})^{\top}\notag\\
&E_3=\sum\limits_{i\neq j} E(\sbf{m}_{1i}\sbf{s}_i^{\top}\sbf{s}_j\sbf{m}_{2j}^{\top})
(\sbf{m}_{1i}\sbf{s}_i^{\top}\sbf{s}_j\sbf{m}_{2j}^{\top})^{\top}=\sum\limits_{i\neq j}\sbf{m}_{1i}E
(\sbf{s}_{i}^{\top}\sbf{s}_{j}\sbf{m}_{2j}^{\top}\sbf{m}_{2j}\sbf{s}_{j}^{T}\sbf{s}_{i})
\sbf{m}_{1i}^{T}\notag\\&\ \ = 
\sum\limits_{i\neq j}\sbf{m}_{1i}\sbf{m}_{2j}^{\top}\sbf{m}_{2j}
E(\sbf{s}_{i}^{\top}\sbf{s}_{j}\sbf{s}_{j}^{T}\sbf{s}_{i}
)\sbf{m}_{1i}^{T} = \sum\limits_{i\neq j}\sbf{m}_{1i}\sbf{m}_{2j}^{\top}\sbf{m}_{2j}
E(\sbf{s}_{i}^{\top}\sbf{s}_{j}
)^2\sbf{m}_{1i}^{T} \notag\\
&= \frac{1}{d}\sum\limits_{i\neq j}
\sbf{m}_{1i}\sbf{m}_{2j}^{\top}\sbf{m}_{2j}\sbf{m}_{1i}^{T}\notag\\
&\ \ = \frac{1}{d} \mbf{M}_1^{\top}\mbf{M}_1\text{tr}(\mbf{M}_2\mbf{M}_2^{\top}) - 
\frac{1}{d}\mbf{M}_{1}^{\top}\mbf{W}_2\mbf{M}_{1}
\end{align}
where $\mbf{W}_2 = \mathrm{diag}\{\sbf{m}_{21}^{\top}\sbf{m}_{21},\cdots, 
\sbf{m}_{2n}^{\top}\sbf{m}_{2n}\}$.
Thus, we have, 
$$E\left[(\mbf{M}_1^{\top}\mbf{A}^{\top}\mbf{A}\mbf{M}_2
 )(\mbf{M}_1^{\top}\mbf{A}^{\top}\mbf{A}\mbf{M}_2)^{\top}\right]
 =\mbf{M}_1^{\top}\mbf{M}_2\mbf{M}_2^{\top}\mbf{M}_1 + 
 \frac{1}{d} \mbf{M}_1^{\top}\mbf{M}_1\text{tr}(\mbf{M}_2\mbf{M}_2^{\top})-
 \frac{1}{d}\mbf{M}_{1}^{\top}\mbf{W}_2\mbf{M}_{1}.$$
\end{proof}

\subsection{An experiment with a higher dimension}
\label{app:d1000}

To assess the performance of the CEPBO algorithm in higher dimensions, 
we conducted simulations using the well-known Hartmann function. 
Specifically, we utilized the Hartmann function with an original space dimension of $D=1000$ and set the embedded space dimension to $d=6$ as well.

\begin{figure}[ht]
    \centering
    \includegraphics[width=1\textwidth]{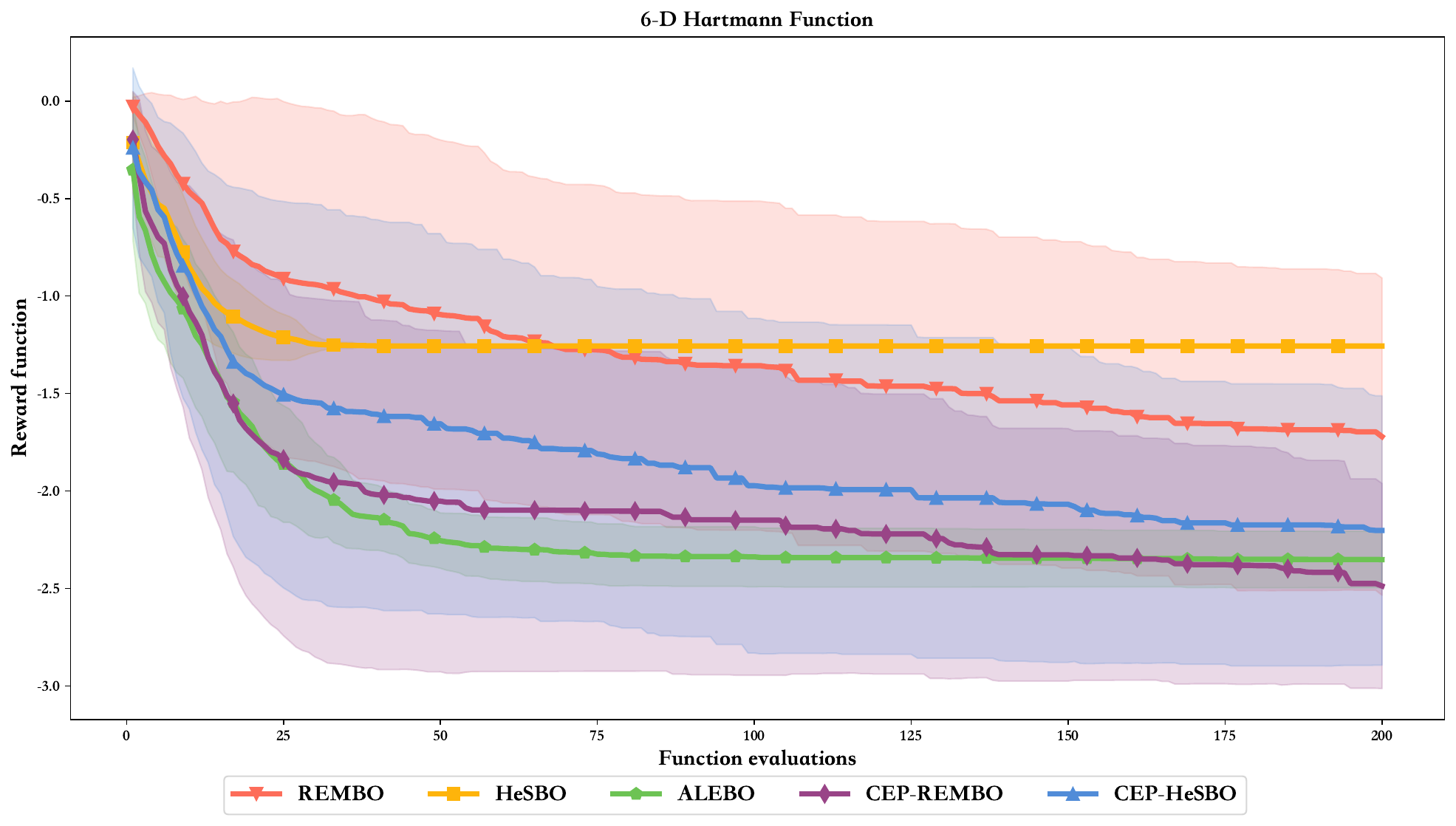}
    \caption{The results of the optimization experiments for Hartmann function.
    }
    \label{fig:hartmann}
\end{figure}
The results are reported in Figure \ref{fig:hartmann}. In this setup, ALEBO, REMBO, and HeSBO were identified as the best-performing configurations. 
The experimental results demonstrate that our proposed algorithm still maintains a certain level of superiority, with the CEP-REMBO algorithm being the optimal one. Additionally, the results indicate that incorporating the CEP projection mechanism can significantly improve the performance of both REMBO and HeSBO algorithms.

\subsection{Details of machine information}
The entire experiment in this paper is programmed in Python and run under the Linux system, using open-source Bayesian optimization libraries such as Ax\cite{bakshy2018ae} and BoTorch\cite{balandat2020botorch} for assistance. The experimental equipment is equipped with 2 AMD EPYC 7601 processors, each with 32 cores and a base clock frequency of 2.2GHz. The experiment is conducted in the form of parallel processing, with different independent repeating experiments distributed to different CPU cores for acceleration. In addition, the system has a memory capacity of 768GB, providing sufficient memory space for large-scale data processing and complex algorithm operation.

\subsection{Modified Schwefel function and Griewank function}
\label{Modified Schwefel function and Griewank function}
To investigate the efficacy of the proposed method within the complex context of high-dimensional optimization problems that entail numerous local minima, we applied a Schwefel function of $D=100$ and a Griewank function of $D=100$.
In the context of the Schwefel function and Griewank function, every dimension qualifies as an effective dimension, hence $d_e=D=100$.
We increased the optimization challenge by altering the Schwefel function and Griewank function. This alteration involved the adjustment of positions within different dimensions where minimum values are attained, thereby making the optimization of the Schwefel function and Griewank function more challenging. 
The modified Schwefel function is:
\begin{equation}
  f(\bsx)=\sum_{i=1}^D \frac{x_i^2}{4000}-\prod_{i=1}^d \cos \left(\frac{x_i}{\sqrt{i}}\right)+1 ,
\end{equation}
Where $b_i \sim \mathcal{N}(0,1)$. We maintain the constancy of $b_i$ values across different independent repeated experiments, while allowing $b_i$ values to vary across different dimensions.

The modified Griewank function can be expressed as:
\begin{equation}
  f(\bsx)=\sum_{i=1}^D w_i(x_i - b_i)^2 -\prod_{i=1}^d \cos \left(\frac{x_i}{\sqrt{i}}\right),
\end{equation}
where $w_i\sim \mathcal{N}(0,1)$ and $b_i\sim \mathcal{N}(0,1)$. We ensure that the values of $w_i$ and $b_i$ remain consistent within various independent repeat experiments, while differing across the several dimensions.

\subsection{Details about Real-World Problems}
\subsubsection{Lunar landing}
In this experiment, our goal is to learn a strategy that controls the lunar lander, so that the lunar lander can minimize fuel consumption and distance from the landing target, while avoiding crashes. This optimization task was proposed by Eriksson\cite{eriksson2019scalable}. The simulation environment of the control task is implemented through OpenAI gym \footnote{www.gymlibrary.dev/environments/box2d/lunar\_lander/}. The state space of the lunar lander includes its coordinates $x$ and $y$, linear velocities $x_v$ and $y_v$, its angle, its angular velocity, and two boolean values indicating whether each leg is in contact with the ground. At any moment, the current controller state can be represented with an $8$-dimensional vector. After obtaining the state vector, the controller can choose one of four possible actions, corresponding to pushing the thrusters left, right, up or none. In the experiment, it can be considered as a $D=12$ optimization problem. Once the parameters are determined, the corresponding rewards can be obtained through in-game feedback. If the lander deviates from the landing pad, it loses rewards. If the lander crashes, it gets an extra $-100$ points. If it successfully controls the lander to stop, it will get an extra $+100$ points. Each leg touching the ground gets $+10$ points. Igniting the main engine gets $-0.3$ points per frame. Each frame starts side engine for $-0.03$ points. The goal of the control task optimization is to maximize the average final reward on a fixed set of 50 randomly generated terrains, initial positions, and speed combinations. We observe that even minor perturbations can have an impact on the simulation.

\subsubsection{Robot pushing}
This paper follows the experimental setup of Wang\cite{wang2017batched}, Eriksson\cite{eriksson2019scalable} et al., and also realizes the simulation of using two robot arms to push two objects in the Box $2 \mathrm{D}$\cite{cattobox2d} physics engine. In the simulation environment, the parameters of the robot arms are simulated to push two objects, and the trajectories of the object movements are recorded at the same time. A total of 14 parameters are used by the two robot arms, which respectively specify the position and rotation of the robot hands, the pushing speed, the moving direction, and the pushing time. The lower bounds for these parameters are $$[-5,-5,-10,-10,2,0,-5,-5,-10,-10,2,0,-5,-5],$$ and the upper bounds are $$[5,5,10,10,30,2 \pi, 5,5,10,10,30,2 \pi, 5,5].$$ The initial positions of the objects are designated as $s_{i 0}$ and $s_{i 1}$, and the end positions as $s_{e 0}$ and $s_{e 1}$. The target positions of the two objects are indicated by $s_{g 0}$ and $s_{g 1}$. The reward is defined as $$r=\left|s_{g 0}-s_{i 0}\right|+\left|s_{g 1}-s_{i 1}\right|-\left|s_{g 0}-s_{e 0}\right|-\left|s_{g 1}-s_{e 1}\right|,$$ namely the distance by which the objects move towards their target positions.

\subsubsection{NAS}
In this paper, referring to the settings of Letham\cite{letham2020re} and others, by parameterizing operations and edges respectively, the optimal architecture search problem in NASBench-101 is set as a continuous high-dimensional Bayesian optimization problem.
Specifically, $L$ different operations are represented by one-hot encoding.

Since two of the seven nodes are fixed as input and output nodes, the remaining five optional nodes, each node corresponds to three different operations, which generate a total of 15  different parameters.
We optimize these parameters in the continuous $[0, 1]$ space.
For each node, we take the "operation" corresponding to the maximum value of the three operations under that node as the "operation" adopted by that node, and use one-hot encoding to represent the specific "operation" used under that node.

Since NASBench-101 uses a $7 \times 7$ upper triangular adjacency matrix to represent edges, it generates a total of $\frac{7\cdot6}{2}=21$ possible edges.
And the five optional vertices can have three different operations, so under this encoding there are about $2^{21} \cdot 3^5 \approx 510M$ unique models,
After removing a large number of unreasonable input and output models and models with more than 9 edges, the search space still has about $423k$ unique models.
We convert these 21 possible edges into 21 binary parameters that are similarly optimized in a continuous $[0, 1]$ space.
We rank the continuous values corresponding to these 21 binary parameters and create an empty adjacency matrix.
Then, we add edges to the adjacency matrix in the percentile order of the 21 binary parameters iteratively, while pruning parts that are not connected to the input or output nodes, until reaching the limit of 9 edges.
Finally, the combination of adjacency matrix parameters (21) and one-hot encoded "operation" parameters (15) constitutes a 36-dimensional optimization space.
The Bayesian optimization algorithm only needs to be optimized in a high-dimensional space with $D=36$, and the boundary constraint is $[-1,1]^{36}$.
Each vector $\bsx \in \mathbb{R}^{36}$ can be decoded into a DAG and lookup evaluated in NASBench-101.

\subsubsection{Rover planning}
To explore the performance of the proposed method in complex high-dimensional optimization scenarios, we considered a two-dimensional trajectory optimization task aimed at simulating detector navigation missions. This optimization task was proposed by Wang\cite{wang2018batched}, and the experimental setup by Wang\cite{wang2018batched} was continued to be used here, with the optimization objective being to maximize the reward function. The problem instance is described by defining the starting position $s$, the target position $g$, and a cost function on the state space. The goal of the problem is to optimize the detector's trajectory on rugged terrain. The trajectory consists of a set of points on a two-dimensional plane, and there are 30 points in this instance, which can be fitted into a B-spline curve, so it is considered a high-dimensional optimization problem with $D=60$. Through a set of trajectories, $\boldsymbol{x} \in[0,1]^{60}$, and a specific cost function, we can calculate the cost of a trajectory $c(\boldsymbol{x})$.

The reward for this problem is defined as 
$$f(\boldsymbol{x})=c(\boldsymbol{x})+\lambda\left(\left|\boldsymbol{x}_{0,1}-s\right|1+\left|\boldsymbol{x}_{59,60}-g\right|_1\right)+b.$$

The reward function is non-smooth, discontinuous, and concave. The four input dimensions involved in the reward function respectively represent the starting and target positions of the trajectory. Set $\lambda=-10, b=5$, any collision with objects along the trajectory will incur a penalty of $-20$, which is the collision cost of the trajectory. Thus, in addition to penalties in the reward function caused by collisions, adverse deviations from the trajectory's starting point will also incur additional penalties.



\end{document}